\documentclass[12pt, final]{l4dc2022}
\title[Learning reversible symplectic dynamics]{%Structure-preserving neural networks for 
Learning reversible symplectic dynamics}
\usepackage{times}
\usepackage{textpos}
\usepackage{fncylab} %formatting of labels
\usepackage{fancyhdr} % page layout
\usepackage{url} % URLs
\usepackage[english]{babel}
\usepackage{amsmath}
\usepackage{mathtools}
\usepackage{graphicx}
\usepackage{dsfont}
\usepackage{epstopdf} %automatically replace .eps with .pdf in graphics
\usepackage{array}
\usepackage{latexsym}
\usepackage{setspace}
\usepackage{subfiles} % Best loaded last in the preamble
\newcommand{\R}[0]{\mathds{R}} % real numbers
\newcommand{\M}[0]{\mathcal{M}} % real numbers
\newcommand{\Ha}[0]{\mathcal{H}} % real numbers
\newcommand{\norm}[1]{\left\lVert#1\right\rVert}
\newcommand{\argmin}[1]{\underset{#1}{\operatorname{arg}\,\operatorname{min}}\;}
\newtheorem{thm}{Theorem}[section]
\newtheorem{dfn}{Definition}[section]
\newcommand{\T}[0]{\mathcal{T}} % natural numbers
\author{
\Name{Riccardo Valperga} \Email{r.valperga@uva.nl}\\
\addr Informatics Institute, University of Amsterdam, The Netherlands\\
\\
\Name{Kevin Webster} \Email{kevin.webster@imperial.ac.uk}\\
\Name{Dmitry Turaev} \Email{d.turaev@imperial.ac.uk}\\
\Name{Victoria Klein} \Email{victoria.klein18@imperial.ac.uk}\\
\Name{Jeroen S. W. Lamb} \Email{jsw.lamb@imperial.ac.uk}\\
\addr Department of Mathematics, Imperial College London, United Kingdom}
\begin{document}
\maketitle
\begin{abstract}%
Time-reversal symmetry arises naturally as a structural property in many dynamical systems of interest. While the importance of hard-wiring symmetry is increasingly recognized in machine learning, to date this has eluded time-reversibility. In this paper we propose a new neural network architecture for learning time-reversible dynamical systems from data. We focus in particular on an adaptation to symplectic systems, because of their importance in physics-informed learning. 
\end{abstract}
\begin{keywords}%
Physics-informed machine learning, time-reversal symmetry, symplectic neural networks, dynamical systems
\end{keywords}

\section{Introduction}\label{sec:intro}
Neural networks are universal approximators: in principle, any sufficiently well-behaved function can be approximated, with arbitrary accuracy, using a neural network \citep{hornik1989multilayer}. In practice, to achieve good approximations we need large datasets and for physical systems data acquisition can be costly. On the other hand, when dealing with physical systems, we very often use models which carry particular mathematical structures. Such structure can be recognised a priori, to guide - and improve the quality of - learning. In this paper, we address the problem of learning time-reversible and symplectic dynamics, motivated by their widespread occurrence in physics. Time-reversibility and symplecticity are fundamental properties that the corresponding learned dynamical system should have \textit{exactly}, rather than approximately.\footnote{In practical terms, this means that any deviation in the structure should be limited by the numerical accuracy of the computational device rather than the (larger) error of the optimisation algorithm.}

Time-reversing symmetry arises naturally as a structural property of many physical systems of interest \citep{lamb1998time}. An elementary example is given by the Newtonian-type second order differential equations
$\ddot{\mathbf{x}} = \mathbf{F}(\mathbf{x},\dot{\mathbf{x}})$: when the forces $\mathbf{F}$ do not depend on the velocities $\dot{\mathbf{x}}$ (or are even functions of the
velocities), the equations do not change after the reversal of time $t \longrightarrow -t$. This implies that if 
$(\mathbf{x}(t),\dot{\mathbf{x}}(t))$ is a solution, then $(\mathbf{x}(-t),-\dot{\mathbf{x}}(-t))$ is also a solution.

In general, a smooth system 
\begin{equation}\label{AutonomousSys}
\dot{\mathbf{x}} = \mathbf{F}(\mathbf{x}),
\end{equation}
of autonomous ordinary differential equations is said to possess a {\em reversing symmetry} if there exists an invertible smooth map $R$ such that
\begin{equation}
\frac{d R(x)}{dt} = -\mathbf{F}(R(x)).
\end{equation}
In this case, we call such system {\em $R$-reversible}. 
Similarly, a {\em discrete-time dynamical system} defined by an invertible map $\T$ is called $R$-reversible if it possesses a reversing symmetry $R$:
\begin{equation}\label{DiscreteRev}
R \circ \T = \T^{-1} \circ R.
\end{equation}
The time-$t$ flow map of an $R$-reversible system of differential equations is $R$-reversible for all $t\in\R$.

{\em Symplectic dynamics} are generated by the \textit{Hamiltonian} differential equations
\begin{equation}\label{ham1}
\dot{\mathbf{q}} = \frac{\partial \mathcal{H}}{\partial \mathbf{p}},  \qquad \dot{\mathbf{p}} = - \frac{\partial \mathcal{H}}{\partial \mathbf{q}},
\end{equation}
where $\mathcal{H}$, the Hamiltonian or energy function, is a given function of the coordinates $\mathbf{q}\in \R^n$ and their conjugate momenta $\mathbf{p}\in\R^n$. This particular class of equations is important because of one of the most basic facts of physics: every \textit{isolated} physical system has a Hamiltonian structure. Hamiltonian systems {\em preserve energy}, i.e., dynamics are constrained to the $(2n-1)$-dimensional manifold $\mathcal{H}(\mathbf{p},\mathbf{q}) = constant$. The flow maps of a Hamiltonian system are symplectic in the following sense.\footnote{For simplicity, we confine the discussion here to systems with a standard Darboux symplectic form.}

A smooth map $\mathcal{T}: \M \longrightarrow \R^{2n}$ (where $\M$ is an open subset of $\R^{2n}$) is said to be symplectic if
\begin{equation}
(d_{x} \T)^{\top} J (d_{x}\T) = J, \quad \forall x \in \M,
\end{equation}
where $d_{x}\T$ is the Jacobian matrix of $\T$ at the point $x$, and $J=\begin{bmatrix}
0 & \mathbb{I}_{n} \\
-\mathbb{I}_{n} & 0 \\
\end{bmatrix}$.

When the Hamiltonian is an even function of the momenta, i.e., it is invariant with respect to the involution
\begin{equation}\label{HamiltonianR}
R: (\mathbf{p}, \mathbf{q}) \longrightarrow (-\mathbf{p}, \mathbf{q}),
\end{equation}
system (\ref{ham1}) is $R$-reversible. The classical example is given by
mechanical systems where the Hamiltonian is the sum of potential and kinetic energies. However, not every
Hamiltonian system is time-reversible (those involving an interaction with a magnetic field are, often, not). Also not every reversible system is Hamiltonian (non-holonomic mechanics provide many examples \citep{gonchenko2020three}).

Note that the canonical reversing symmetry
$R(p,q)=(-p,q)$ is an involution (i.e., $R^2=id$) and is also {\em anti-symplectic}, i.e., it satisfies 
$R^\top J R = - J$. It is indeed natural to consider Hamiltonian 
systems with (anti-)symplectic time-reversal symmetries. In the examples
discussed in this paper, reversing symmetries are always anti-symplectic, but our results also hold for symplectic reversing symmetries. The structure defined by a combination of symplecticity
and an involutory reversing (anti-)symplectic symmetry is henceforth called {\em reversible
symplectic structure}.\footnote{Non-involutory reversing symmetries may also be considered, but they imply the existence of additional time-preserving symmetries, leading to reversible equivariant (symplectic) settings, which will not be considered in this
paper.}

The importance of geometric  structure - such as symmetry or symplecticity - has began to be recognized in the context of learning dynamical systems from observations \citep{burby2020fast, jin2020symplectic}, but - to date - this has eluded time-reversing symmetries. In this paper, we propose a new structure-preserving neural network for learning reversible dynamical systems. We focus in particular on learning reversible symplectic dynamics for their importance in physics-informed learning.

We implement our method for several examples of chaotic behavior including the paradigmatic H\'enon-Heiles system and the periodically forced pendulum. Both in qualitative and quantitative aspects, we demonstrate improved performance and accuracy of our structure-preserving network, over the existing symplectic learning algorithms that are not time-reversible.
%\end{document}

%\section{Reversible symplectic dynamics}
%\input{2_TimeReversibleSymplecticSystems}

\section{Setup}\label{PINNS}
%\begin{document}
As usual in a supervised learning setting, the objective is to learn a map $\mathcal{T}: \M \longrightarrow \M$ representing the evolution of a dynamical system, from observations: given a set $\left\{(x_{i}, y_{i} = \T(x_{i})) \right\}_{i=1}^{N} \subset \M$, the goal is to approximate $\T$ with some map $\hat{\T} \in \mathbf{H}$, where $\mathbf{H}$ a given function space which we refer to as the \textit{hypothesis space}. We aim to find $\hat{\T}$ approximating $\T$ with best fit, in the sense that
\begin{equation}
\hat{\T} = \argmin{\T^{*} \in \mathbf{H}} \frac{1}{N} \sum_{i=1}^{N}d(\T(x_{i}), \T^{*}(x_{i})),
\end{equation}
where $d$ denotes the Euclidean metric $d(x, y) =  \sqrt{\norm{x - y}^{2}}$.

% \subsection{Supervised learning setting}
% In the supervised learning setting, we deal with regression problems and predictions. We have a fixed input space $\mathcal{X}$, a fixed output space $\mathcal{Y}$, and a dataset $\mathcal{D} = \{(x_{i}, y_{i}) \in \mathcal{X} \times \mathcal{Y} \}$ that is made of pairs of points from the input and output space, possibly affected by some random noise, such that $y_{i}$ is the image of $x_{i}$ under some target function $\T^{*}: \mathcal{X} \longrightarrow \mathcal{Y}$. The goal is to learn/approximate the function $\T^{*}$ using the information contained in $\mathcal{D}$.
% \subsection{Physics-informed restricted spaces and universal approximation}
In the context of physics-informed learning, there may be a priori knowledge about the structure of the dynamical system that is to be learned. In this case, the hypothesis space should possess this structure intrinsically. Natural hypothesis spaces that approximate with arbitrary accuracy are infinite dimensional. The constructive approach is to build a countable family of hypothesis subspaces $\{ \mathbf{H}_{\ell} \}_{\ell \in \mathbb{N}}$, 
with $\mathbf{H}_{\ell}\subset\mathbf{H}_{\ell+1}$ for all $\ell\in\mathbb{N}$, such that $\lim_{\ell \to\infty} \mathbf{H}_{\ell} = \mathbf{H}$ in an appropriate topology. In other words, $\cup_{\ell} \mathbf{H}_{\ell}$ universally approximates $\mathbf{H}$. When the learned function is represented by a neural network, the number $\ell$ is indicative of the number of degrees of freedom in the network. In practice, one settles for a certain $\mathbf{H}_{\ell^{*}}$ with $\ell^{*}$ large enough, so that targets are sufficiently well approximated. Key to a structure-preserving approach is that $\mathbf{H}_\ell$ has the desired structure {\em exactly} for all $\ell \in \mathbb{N}$. The objective of this paper is to propose choices of $\{ \mathbf{H}_\ell \}_{\ell \in \mathbb{N}}$ that are reversible or reversible symplectic.

Neural networks architectures with hardwired symplectic structure have been proposed before. \citep{NEURIPS2019_26cd8eca, desai2021port} learn a Hamiltonian function using a neural network. This requires numerical integration and training datasets with points that are sufficiently close to each other to estimate derivatives. \cite{jin2020symplectic, burby2020fast} develop algorithms for learning symplectic maps (including time-shifts by the flow of Hamiltonian systems). In either approach, time-reversibility was not yet taken into account.

\section{Reversible symplectic neural networks}
In this section we introduce reversible and reversible symplectic neural networks, which we subsequently show to be universal approximators for
reversible and reversible symplectic dynamical systems.
% \documentclass[main.tex]{subfiles}
% \begin{document}
\subsection{Reversible neural networks}\label{sec:revnn}
Consider the target space of $R$-reversible diffeomorphisms with an involutary
symmetry $R$ (i.e., $R \circ R = id_{\R^{d}}$):
\begin{equation}
\mathbf{H} = \left\{ \T \in \mathrm{Diff}(\R^d) \quad | \quad  \T = R \circ \T^{-1} \circ R \right\}.
\end{equation}
We aim to find a family of hypothesis spaces $\{\mathbf{H}_\ell\}_{\ell\in\mathbb{N}}$ of $R$-reversible maps (i.e.~$\mathbf{H}_\ell\subset\mathbf{H}$ for all $\ell\in\mathbb{N}$) such that $\lim_{\ell \to\infty} \mathbf{H}_{\ell} = \mathbf{H}$. Furthermore, the functions in $\mathbf{H}_\ell$ are to be parameterised in such a way that we can use common 
gradient-descent based optimization algorithms to find the best approximation to the unknown target function $\T\in \mathbf{H}$. We construct ${\mathbf{H}_\ell}$ using compositions of neural network-like maps. 

Note that, in general, compositions of $R$-reversible maps are not $R$-reversible. However, given any set of invertible maps $\{f_{n}\}_{n=1}^\ell$ and letting
 $\hat{f}_{i} := R \circ f_{i}^{-1} \circ R$, it is readily verified that the following composition is $R$-reversible
\begin{equation}\label{RevComp1}
F:=\hat{f}_{1} \circ \hat{f}_{2} \circ \dots \circ \hat{f}_{\ell} \circ f_{\ell} \circ f_{\ell-1} \circ \dots \circ f_{1}.
\end{equation}
%Indeed, used the fact that $R$ is an involution, we obtain
%%We can see that this composition is indeed $R$-reversible as
%\begin{equation}
%%\begin{split}
%RFR = %R \left(R f_{1}^{-1}  \dots f_{N}^{-1} R  f_{N} \dots f_{1} \right) R =
%f_{1}^{-1} \dots f_{N}^{-1} R  % f_{N}^{-1}  R^{-1} 
%f_{N} \dots f_{1} R = 
%f_{1}^{-1} \dots f_{N}^{-1} R  % f_{N}^{-1}  R^{-1} 
%\hat{f}_{N}^{-1} \dots \hat{f}_{1}^{-1} = F^{-1}.
%%\end{split} 
%\end{equation}

We propose $\mathbf{H}_\ell$ to consist of functions of the form \eqref{RevComp1}, where the functions $\{f_n\}_{n=1}^\ell$ are Real NVP bijective layers \citep{dinh2016density} which have the important useful property that they are exactly (and easily) invertible.
% that possesses the necessary universal approximation property, discussed later, and that, most importantly, unlike neural network layers, is analytically invertible. 
Real NPV bijective layers are defined as follows:
\begin{dfn}
Given a $d$-dimensional input $x$, two functions $s,t : \R^{d'} \to \R^{d-d'}$, with $d'<d$, the output function $y_{1:d} = y(x_{1:d})$ of a Real NVP is
\begin{equation}\label{Bijector}
\left\{\begin{split}
&y_{1:d'} = x_{1:d'} \\
&y_{d'+1:d} = x_{d'+1:d} \odot \text{exp}\left( s(x_{1:d'}) \right) + t(x_{1:d'}),
\end{split}\right.
\end{equation}
where $\odot$ is the element-wise product, and $y_{1:d} : = (y_{1}, y_{2}, \dots, y_{d})$.
\end{dfn}
\subsection{Symplectic reversible neural networks}\label{sec:revsympnn}
In analogy to Section~\ref{sec:revnn}, we consider $R$-reversible symplectic diffeomorphisms on a $\mathbb{R}^{2n}$ 
as the target space:
%. We consider the following hypothesis space
\begin{equation}\label{HypothesisRevSymp}
\tilde{\mathbf{H}}%^{\mathcal{H}} 
= \left\{ \T \in \mathrm{Diff}(\mathbb{R}^{2n})~|~ \T = R \circ \T^{-1} \circ R ~ \mbox{ and } ~ \left( D_x  \T \right)^{T} J \left( D_{x}  \T \right) = J,~\forall x \in  \mathbb{R}^{2n}\right\}.
\end{equation}
%To construct the family of hypothesis spaces 
We propose to construct the hypothesis spaces $\tilde{\mathbf{H}}_\ell$ 
%such that $\overline{\lim_{l \to\infty} \mathbf{H}^{\mathcal{H}}_{l}} = \mathbf{H}^{\mathcal{H}}$ we propose to use
using compositions of the form \eqref{RevComp1}, where the maps $f_i$ are polynomial H\'enon maps. Polynomial H\'enon maps are symplectic transformations on $\mathbb{R}^n\times\mathbb{R}^n$, $(x, y) \mapsto (\bar{x}, \bar{y})$, with
\begin{equation}\label{HenonMap}
\left\{
\begin{split}
&\bar{x} = y\\
&\bar{y} = -x + \nabla V(y),
\end{split}
\right.
\end{equation}
with $V : \R^{n} \to \R^{n}$ polynomial. We note here that, as required,  the resulting composition \eqref{RevComp1} is symplectic if $R$ is symplectic or anti-symplectic.

%The reason why we propose to use H\'enon maps is due to their universal approximation capabilities discussed later. 
%The $R$-reversible, symplectic hypothesis space is then
%\begin{dfn}
%\begin{equation*}\label{ReversibleHenon}
%\begin{split} 
%&\mathbf{H}^{\mathcal{H}}_l = \big\{ \hat{h}_{(w_{1}, \eta_{1})} \circ \hat{h}_{(w_{2}, \eta_{2})} \circ \dots \circ \hat{h}_{(w_{L}, \eta_{L})} \circ h_{(w_{L}, \eta_{L})} \circ h_{(w_{L-1}, \eta_{L-1})} \circ \dots \circ h_{(w_{1}, \eta_{1})} \\
%&\quad L \in \N, \quad \hat{h}_{i} = R \circ h_{i}^{-1} \circ R, \quad h_{(w_{i}, \eta_{i})} \text{ are polynomial H\'enon maps}, \text{  for  }  i = 1, \dots , L \big\}.
%\end{split}
%\end{equation*}
%\end{dfn} 
%Since we are using the reversing symmetry $R$ itself twice, for the composition to be symplectic the reversing symmetry must be an (anti)-symplectic transformation, which is the case for the time-reversing symmetry of Hamiltonian systems \eqref{HamiltonianR}. 

For practical implementation, it is useful that, like Real NVPs, the inverses of polynomial H\'enon maps admit convenient and straightforward analytical expressions. 
%The trainable weights of a polynomial H\'enon map $(w, \eta)$ are the coefficients of the polynomial and the constant shift. 

%Consider, as an example, a 2-dimensional, second-order polynomial H\'enon map. The forward and inverse are simply
%\begin{equation}\label{ForwardHenonExample}
%\begin{split}
%&\bar{x} = y + \eta \qquad \qquad \qquad \qquad \qquad \qquad \text{ } x = -\bar{y} + w_{0} + \bar{x} - \eta + w_{1}(\bar{x} -\eta) + w_{2}(\bar{x}-\eta)^{2}\\
%&\bar{y} = -x + w_{0} + w_{1}y + w_{2}y^{2}, \qquad \quad \quad \text{ } y = \bar{x} - \eta.
%\end{split}
%\end{equation}

\subsection{Universal approximation}
We finally establish the universal approximation properties of the hypothesis spaces $\mathbf{H}_\ell$ and $\tilde{\mathbf{H}}_\ell$, proposed in Sections \ref{sec:revnn} and \ref{sec:revsympnn}.  Note that the composition \eqref{RevComp1} can be expressed as
%is $R$-reversible does not ensure that \textit{any} $R$-reversible map can be decomposed this way.
%Here we prove that this is indeed the case: any $R$-reversible map $\T$ can be decomposed as 
\begin{equation}\label{RevDecomposition}
\T = R \circ g^{-1} \circ R \circ g,
\end{equation}
with $g:=f_{N} \circ f_{N-1} \circ \dots \circ f_{1}$. 
Let us first show that reversible maps indeed admit such a decomposition under mild assumptions. Let $R$ be a linear involution. Recall that an orientation-preserving $R$-reversible $C^2$-diffeomorphism
$\T$ of a $d$-dimensional ball into $R^d$ is \textit{smoothly isotopic to the identity}, i.e., there exists a smooth family of $R$-reversible diffeomorphisms $f: \R^{d} \times [0,1] \to \R^{d}$ such that $f(x, 0) = x$, and $f(x, 1) = \T(x)$, for all $x$. Moreover, if $\T$ is symplectic, that the diffeomorphism
in the family $f$ are also symplectic.
\begin{thm}\label{thm:dec}
Let $\T$ be an $R$-reversible diffeomorphism, with $R$ being
a linear involution. Let $\T$ be smoothly isotopic to the identity. Then there exists a diffeomorphism  $g: \R^{d} \to \R^{d}$, such that $\T = R\circ g^{-1} \circ R \circ g$.
If $\T$ is symplectic, then $g$ can be chosen symplectic.
\end{thm}
% We make the further assumption that $F(0)= 0$ and $F'(0) = $ id. Consider
% \begin{equation}
% \begin{split}
% &f_{t} : U \rightarrow \R^{2n} \\
% &x \mapsto f_{t}(x) = \frac{1}{t}F(t x).
% \end{split}
% \end{equation}
% This is known as the Alexander's isotopy \cite{alexander1923deformation}. $R$-reversibility of $F$, that is $F^{-1} = R F R$, implies 
% \begin{equation}\label{IsotopyReversibility}
% f^{-1}_{t} = R f_{t} R.
% \end{equation}
% We can see this by explicitly writing the inverse:
% \begin{equation}
% f_{t}^{-1} = m_{\frac{1}{t}} \circ F^{-1} \circ m_{t},
% \end{equation}
% where $m_{t}$ is the multiplication by $t$.
\begin{proof}
Define the non-autonomous vector field $h_s$ by the rule $h_s(f(x,s))=\partial_s f(x,s)$, where $f$ is the isotopy from the condition of the theorem. It is well-known that if the isotopy is symplectic, then $h$ is Hamiltonian. Define the vector field
$$y_s=\frac{1}{2} R h_s \circ R.$$
We claim, that the inverse of the sought map $g$ is the time-$1$ map defined along the orbits of the non-autonomous vector field $y_s$.

Indeed, it is enough to show that, for all $s\in[0,1]$,
$$R f_s\circ g_s = g_s \circ R,$$
where $g_s$ is the time-$s$ map defined by the orbits of $y_s$
(we denote $f_s=f(\dot,s)$). Differentiating this with respect to $s$, we find that it is enough for the vector field $y_s$ to satisfy the following identity:
\begin{equation}\label{ggids}
R h_s \circ f_s + (R D_x f_s) y_s= y_s \circ R\circ f_s.
\end{equation}
By the reversibility,
$$f_s \circ R \circ f_s = R.$$
Differentiating this equality with respect to $s$, we find
$$h_s\circ R +(D_{R f_s(x)}f_s) R h_s\circ f_s =0,$$
and (\ref{ggids}) indeed follows.

By construction, if $h_s$ is Hamiltonian, then $y_s$ is Hamiltonian too, so $g$ is symplectic, as required.
\end{proof}

Theorem~\ref{thm:dec} implies that the problem of approximating $\T$ with arbitrary accuracy is reduced to the problem of approximating $g$ with arbitrary accuracy. In the general (non-symplectic) reversible case, we can approximate any diffeomorphism $g$ by Real NVP bijective layers in $L^p$
according to the result of \citep{teshima2020coupling}. A stronger, $C^\infty$-approximation result, follows from \citep{Tu2015}. When $g$ is symplectic, it can be 
$C^\infty$-approximated by compositions of polynomial H\'enon maps according to 
\citep{turaev2002polynomial}.

\section{Numerical experiments}
% \documentclass[main.tex]{subfiles}
% \begin{document}
Finally, we present some results concerning the approximation of Poincar\'e return maps of time-reversible Hamiltonian systems. We consider a periodically driven pendulum and the H\'enon-Heiles system. In each of these examples, the dynamical systems are defined as solutions of explictly given differential equations. However, the Poincar\'e maps that we aim to learn do not admit explicit expressions that can be derived directly from the differential equations. 
Hence, it is a natural objective to learn such maps using a neural network.

\subsection{The H\'enon-Heiles system}
Consider a particle in $\mathbb{R}^2$, with phase space coordinates $(x,y,p_{x},p_{y})\in\R^{4}$, whose equations of motion are generated by the Hamiltonian
\begin{equation}\label{HenonHeilesHam}
\Ha(x,y,p_{x},p_{y}) = \frac{1}{2}(p_{x}^{2} + p_{y}^{2} + x^{2} + y^{2}) + \lambda(x^{2}y - \frac{y^{3}}{3}).
\end{equation}
This dynamical system is known as the H\'enon-Heiles System, after Michel H\'enon and Carl Heiles who introduced it in their celebrated 1964 paper \citep{henon1964applicability}, as a model for the motion of a star in an axisymmetric potential. For a given value of the energy $E\in\R$ the orbits are constrained to 3-dimensional level sets of the Hamiltonian $\mathcal{E} = \Ha ^{-1}(E)$. 

The aim is to study the dynamics of the H\'enon-Heiles system using a so-called Poincar\'e return map.
%s of Hamiltonian systems, let $\mathcal{H}: \M \longrightarrow \R$ be a time-independent Hamiltonian defined on a smooth, $(2n)$-dimensional manifold $\M$. For such systems, energy is conserved and the flow occurs on the $(2n -1)$-dimensional submanifold $\mathcal{E}$ on which $\mathcal{H} = E$.
% Now consider another $(2n-1)$-dimensional submanifold $\mathcal{Q}$ that is nowhere parallel to the flow. 
A Poincar\'e section $\mathcal{P}$ is a $(2n-2)$-dimensional submanifold of $\mathcal{E}$ that is everywhere transversal to the flow. 
%On such a Poincar\'e section we define the Poincar\'e return map as follows:
Let $\phi_{t}$ be the flow of the Hamiltonian system and $\mathcal{P}$ a Poincar\'e section. Then the Poincar\'e return map $\T : \mathcal{P} \longrightarrow \T(\mathcal{P}) \subseteq \mathcal{P}$ is defined as
\begin{equation*}
\T(x) = \phi_{t_{0}}(x),
\end{equation*}
with $t_{0}>0$ least so that $\phi_{t_0}(x) \in \mathcal{P}$. If the section is smooth then  $\T : \mathcal{P} \to\T(\mathcal{P})$ is a 
symplectic diffeomorphism. For more details on Poincar\'e return maps see \citep{BROER20101}.

It is readily verified that for the H\'enon-Heiles system
\begin{equation}\label{Psec}
\mathcal{P} = \{(x, y, p_{x}, p_{y}) \in \mathcal{E} \text{  } \rvert \text{  } x=0, p_{x}>0\}
\end{equation}
defines a Poincar\'e section on $\mathcal{E}$. In Fig.~\ref{PoincareHenonHeiles}(a-c) we present some \textit{Poincar\'e plots} (portraits containing selected orbits of initial conditions under the Poincar\'e map, illustrating the dynamics) for different values of the energy $E$. %In what follows we will call Poincar\'e plot a picture that shows many successive iterations, from different initial conditions, of a Poincar\'e map.
\begin{figure*}[h]
\centering
\subfigure[]{\includegraphics[width=0.3\linewidth]{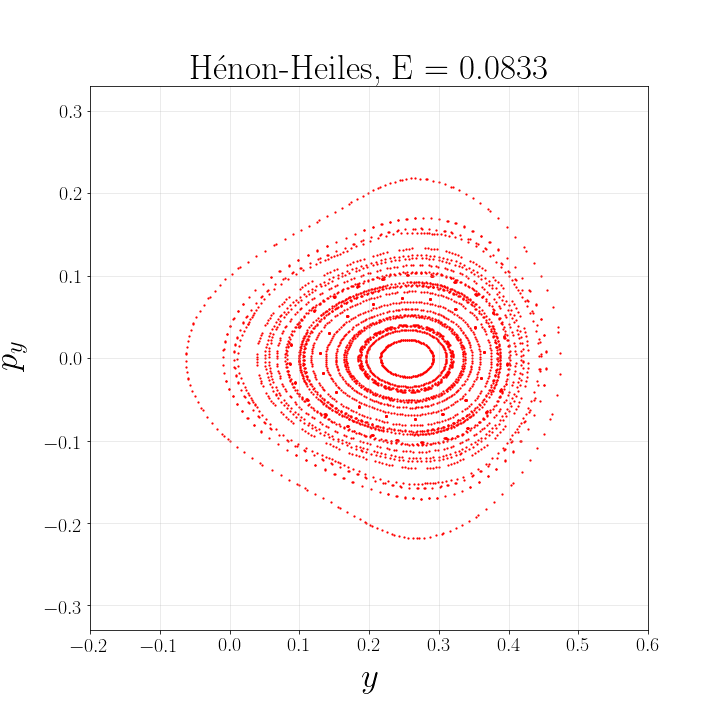}}
\subfigure[]{\includegraphics[width=0.3\linewidth]{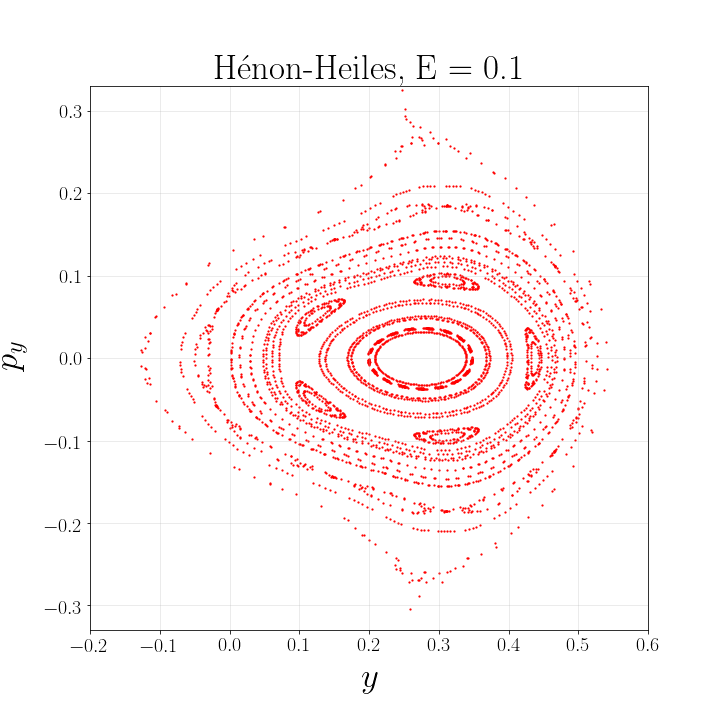}}
\subfigure[]{\includegraphics[width=0.3\linewidth]{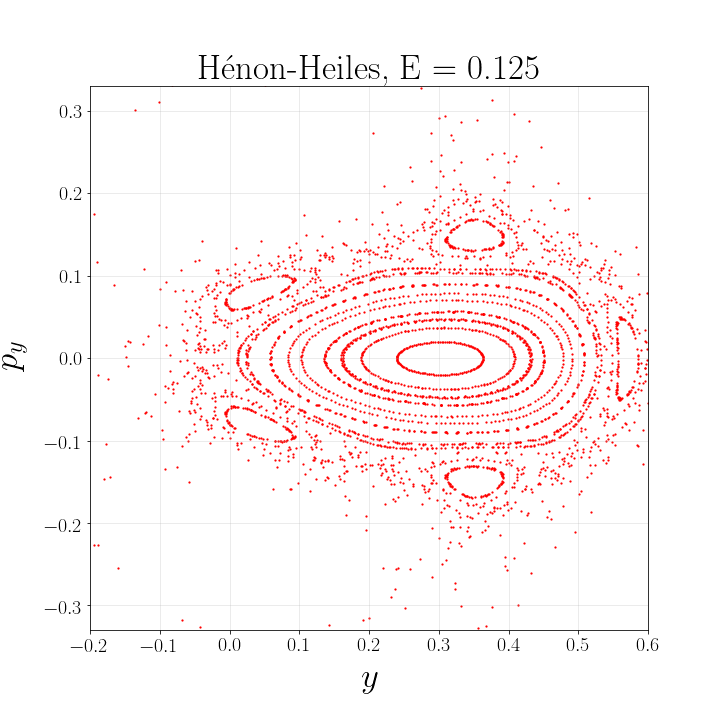}}\\
\subfigure[]{\includegraphics[width=0.3\linewidth]{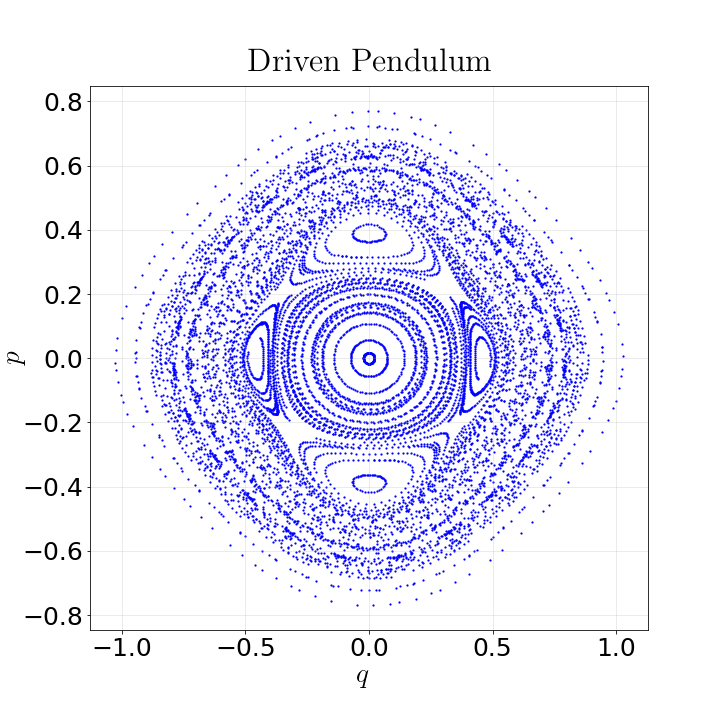}}
\caption{\textit{(a), (b), (c)} Poincar\'e plots of the H\'enon-Heiles system for different values of the energy $E$, and surface of section \eqref{Psec}. %$x=0$, $p_{x}>0$. 
\textit{(d)} Poincar\'e plot of the periodically driven pendulum.}
\label{PoincareHenonHeiles}
\end{figure*}
The H\'enon-Heiles Hamiltonian \eqref{HenonHeilesHam} satisfies the property that $\Ha(x,y,p_x,p_y)=\Ha(x,y,-p_x,-p_y)$ which yields its dynamics reversible with respect to the reversing symmetry $R(x,y,p_x,p_y)=(x,y,-p_x,-p_y)$. As a consequence, it follows that the Poincar\'e return map $\T$ on $\mathcal{P}$ is also reversible, with respect to the involution $R_\mathcal{P}:(y,p_y)\to(y,-p_y)$. This time-reversal symmetry gives rise to the reflection symmetry $R_\mathcal{P}$  in the Poincar\'e plots in 
Fig.~\ref{PoincareHenonHeiles}(a-c). 

Poincar\'e maps provide important information about the dynamics. However, extensive studies of such maps obtained directly from the differential equations by numerical integration is often impractical due to the expense of the computations involved, tracking trajectories in the ambient space between hitting the surface of section. Training a neural network to learn a Poincar\'e map may be computationally expensive as well, but once trained it is fast to iterate. 

\subsection{Periodically driven pendulum}
Consider the time-periodic Hamiltonian
\begin{equation}\label{PerturbedPendulum}
\Ha(p, q, t) = \frac{1}{2}p^{2} - \nu ^{2}cos(q) - \lambda\left[0.3pq \sin(2t) + 0.7pq\sin(3t)\right].
\end{equation}
This Hamiltonian represents a simple pendulum with natural frequency $\nu$ driven by a $2\pi$-periodic force. The time-$2\pi$ map of this system defined by the flow $\Phi_{2\pi}$ is usually also referred to as a Poincar\'e map as one could view time as an augmented phase space coordinate and  $t=2\pi$ can be considered a surface of section in the extended phase space $(p, q, t)$. A Poincar\'e plot of this map is depicted in Fig.~\ref{PoincareHenonHeiles}(d).

It can be shown that the time-$2\pi$ map $\Phi_{2\pi}$ is reversible symplectic with reversing symmetry $R(q,p)=(q,-p)$, which is also reflected by the observed symmetry in the Poincar\'e plot in Fig.~\ref{PoincareHenonHeiles}(d). 

\subsection{Dataset and hypothesis spaces for experiments}
For both systems the dataset is of form $\{(x_{i}, y_{i}), \mathcal{T}(x_{i}, y_{i}) \}_{i=1}^{N}$ where $\T$ is the Poincar\'e map obtained via numerical integration. For both the pendulum and the H\'enon-Heiles system, the optimizer has a small number of points available, i.e., $N=100$ and $N=300$ respectively.

We include the most significant experiments for three different physics-informed hypothesis spaces and the hypothesis space of plain MLPs:
\begin{itemize}
    \item $\mathbf{H}_{NN}$: the set of MLPs, with depth $6$, and a total of $\approx 5000$ trainable parameters.
    \item $\mathbf{H}_{R}$: A function space $\mathbf{H}_{\ell}$ of $R_\mathcal{P}$-reversible diffeomorphisms with 6 different, appropriately masked, Real NVP bijective layers of form \eqref{Bijector} in which the functions $s$ and $t$ are represented by shallow MLPs.
    \item $\mathbf{H}_{HR}$: A function space $\tilde{\mathbf{H}}_{\ell}$ of  $R$-reversible compositions of polynomial H\'enon maps of degree 4, with depth $25$.
    \item $\mathbf{H}_{SN}$: The set of symplectic neural networks (SympNets \citep{jin2020symplectic}) of depth $18$, with $8$ linear modules for each activation module.
\end{itemize}
The hyperparameters of these hypothesis spaces have been chosen so that they all have comparable numbers of trainable parameters and training times. All the experiments are implemented using Keras \citep{chollet2015keras} and have been run on a single Nvidia RTX 2080. The source code can be found agithub.com/Ricvalp/SymplecticTimeReversibleNN. The models have been trained with Adam optimizer and a decaying learning rate; unless stated otherwise, the dataset has been divided into training/validation sets with a 9:1 ratio. Figure \ref{Loss} reports the losses during training.
\subsection{Results}
To compare performances we have evaluated the trained evolution maps in the following way. First, we compare the resulting Poincar\'e plots. From these figures it is clear that evolution maps from intrinsically time-reversible hypothesis spaces result in Poincar\'e plots that look very similar to the ground truth obtained using numerical integration (Fig. \ref{PoincareHenonHeiles}). We then calculate the average Euclidean norm between a few ground truth trajectories and the predictions. Poincar\'e plots and errors are shown in Fig. \ref{PoincarePlots} for the driven pendulum, and in Fig. \ref{PoincarePlots1} for the H\'enon-Heiles system.
\begin{figure}[h]
\centering
\subfigure[]{\includegraphics[width=0.49\linewidth]{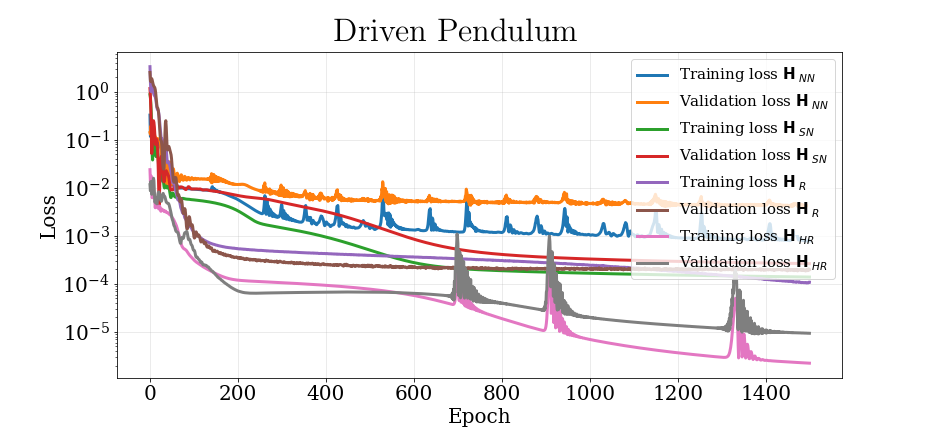}}
\centering
\subfigure[]{\includegraphics[width=0.49\linewidth]{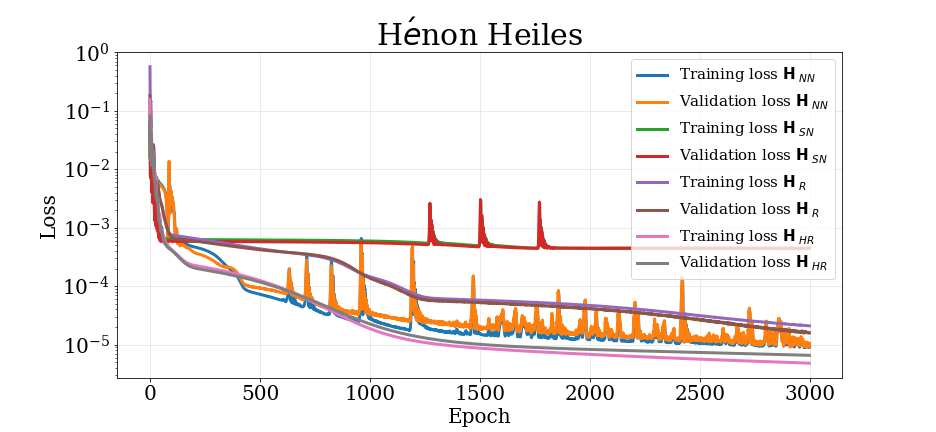}}
\caption{Training and validation loss versus epoch for the four hypothesis spaces and the two systems.}
\label{Loss}
\end{figure}

\begin{figure}[h]
\centering
\subfigure[]{\includegraphics[width=0.31\linewidth]{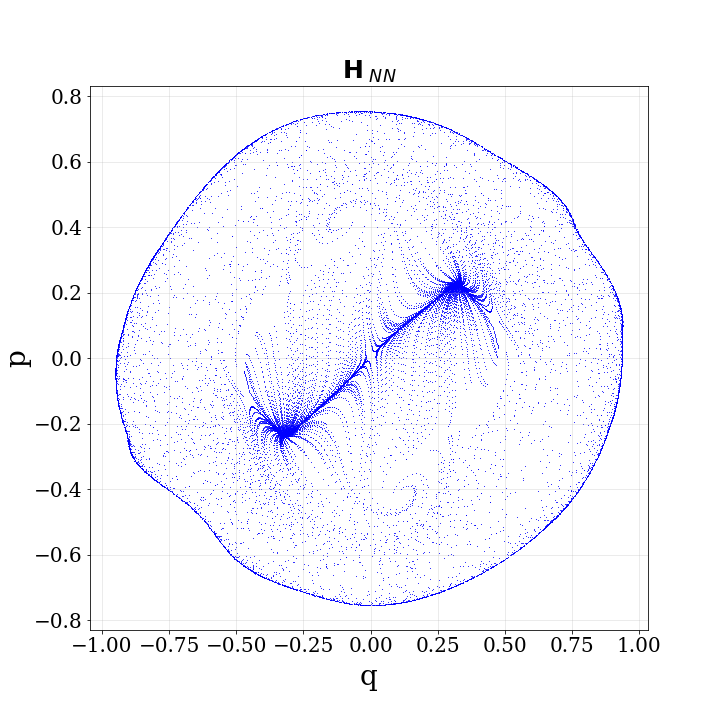}} 
\subfigure[]{\includegraphics[width=0.31\linewidth]{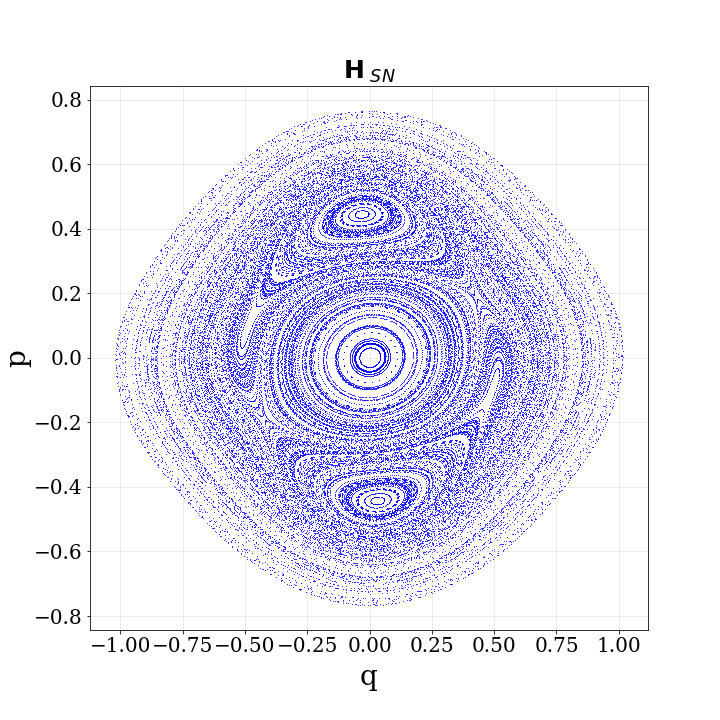}}
\subfigure[]{\includegraphics[width=0.31\linewidth]{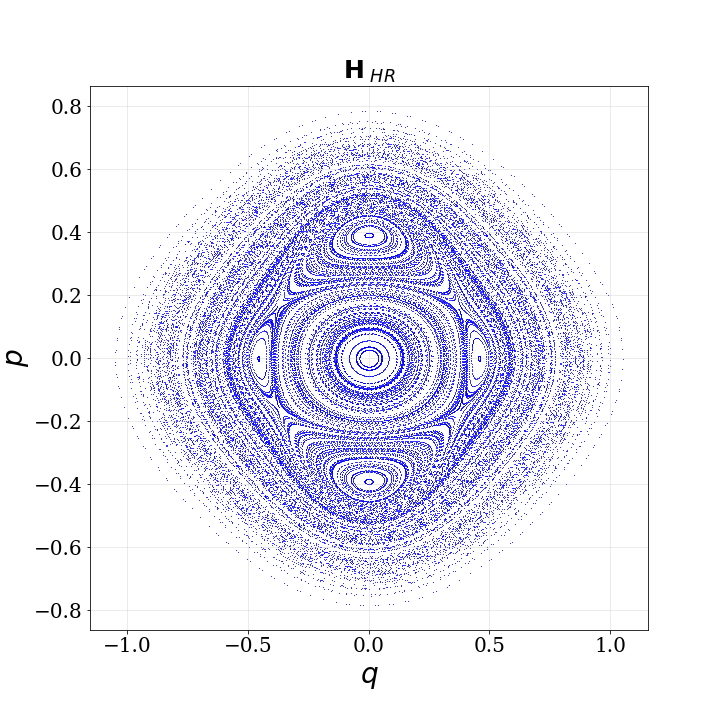}}\\
\subfigure[]{\includegraphics[width=0.31\linewidth]{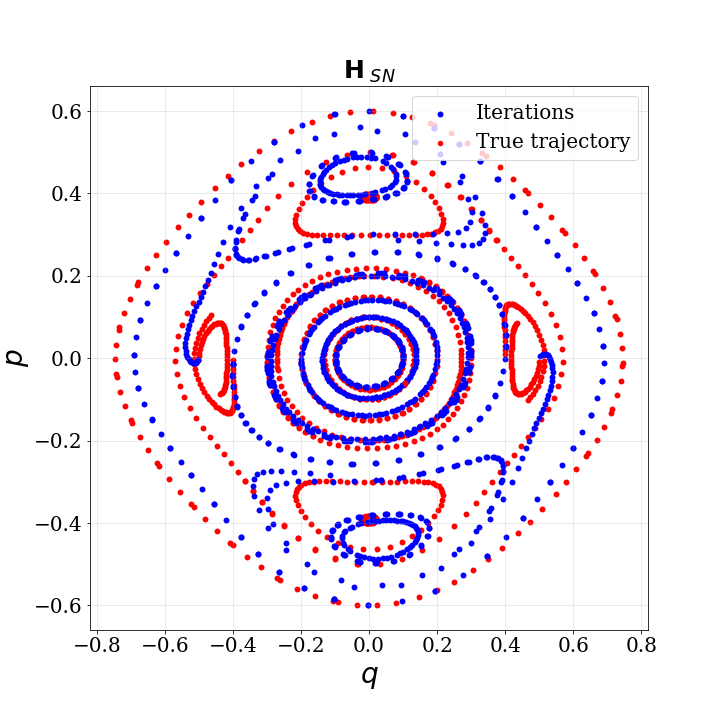}} 
\subfigure[]{\includegraphics[width=0.31\linewidth]{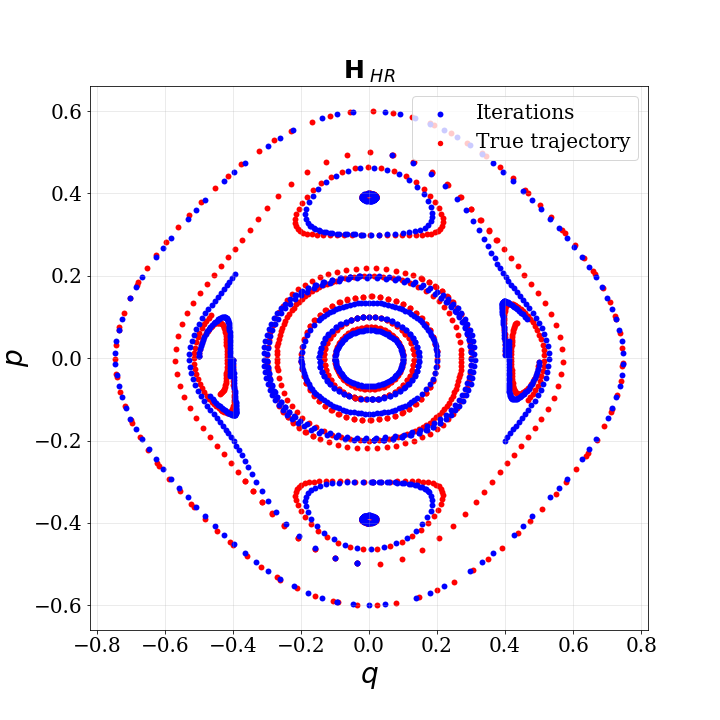}}
\subfigure[]{\includegraphics[width=0.31\linewidth]{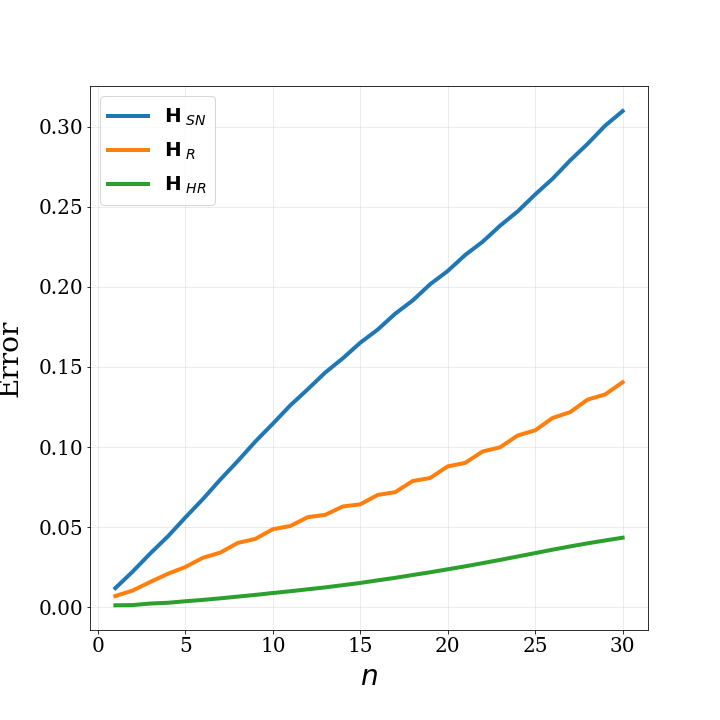}}
\caption{\textbf{Periodically driven pendulum}. Top: Poincar\'e plots obtained iterating the trained evolution maps from $\mathbf{H}_{NN}$,  $\mathbf{H}_{SN}$, and $\mathbf{H}_{HR}$. Bottom: a few predicted trajectories and the average prediction error, calculated as the average euclidean norm of the difference vs. the iteration $n$.}
\label{PoincarePlots}
\end{figure}
\begin{figure}[h]
\centering
\subfigure[]{\includegraphics[width=0.31\linewidth]{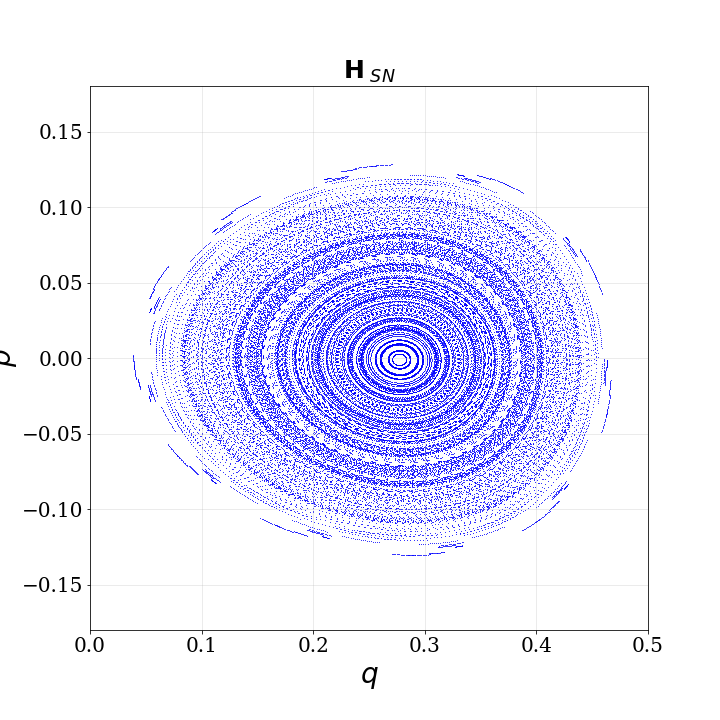}} 
\subfigure[]{\includegraphics[width=0.31\linewidth]{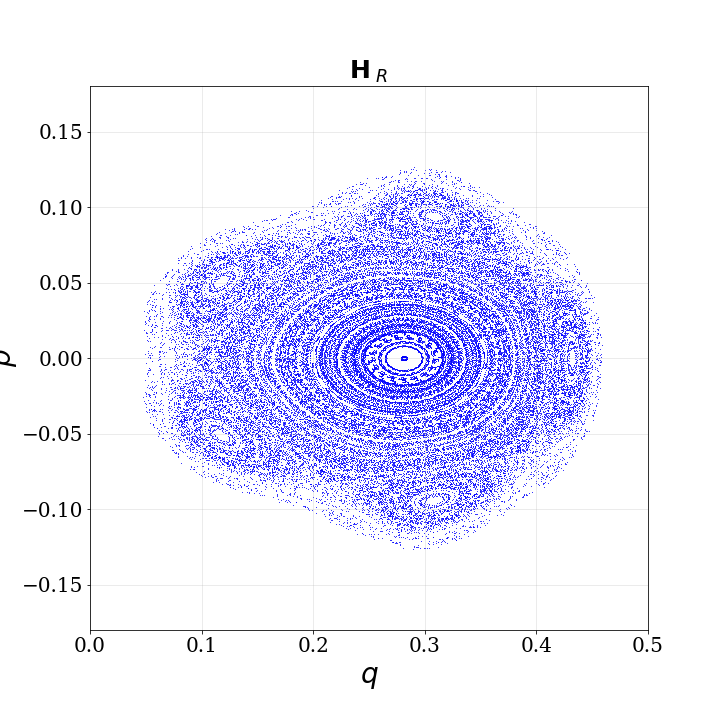}}
\subfigure[]{\includegraphics[width=0.31\linewidth]{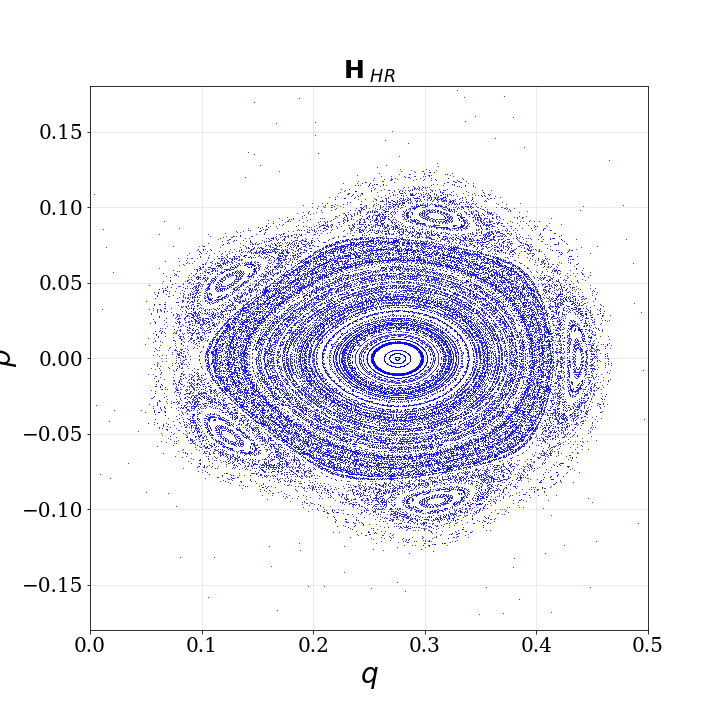}}
\subfigure[]{\includegraphics[width=0.31\linewidth]{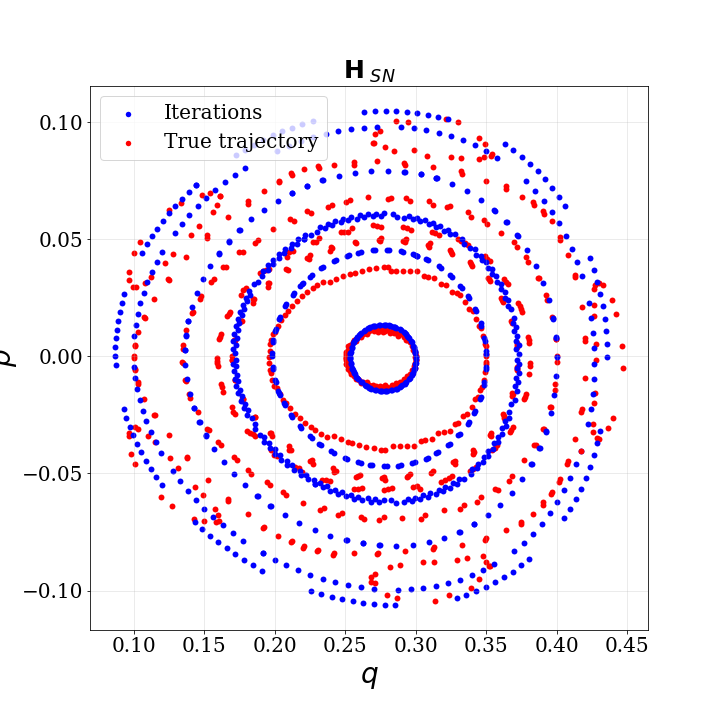}} 
\subfigure[]{\includegraphics[width=0.31\linewidth]{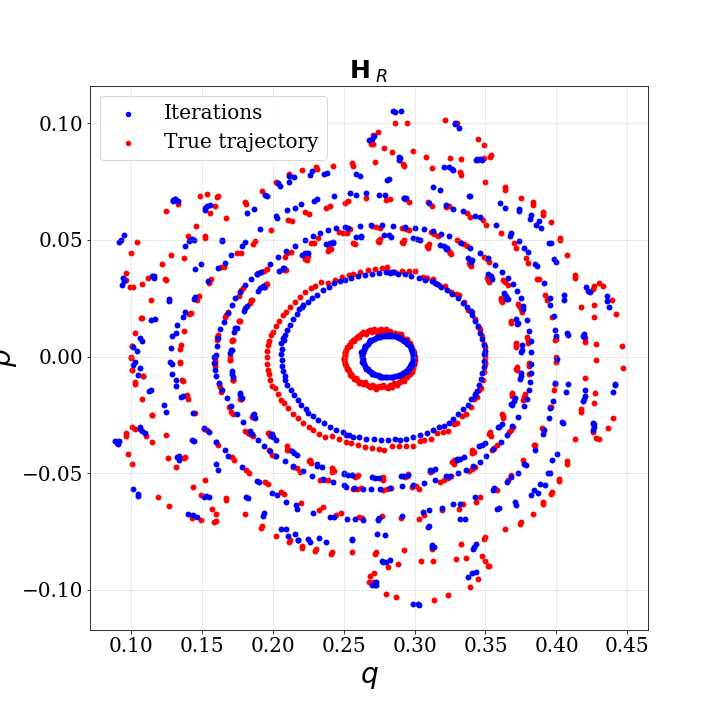}}
\subfigure[]{\includegraphics[width=0.31\linewidth]{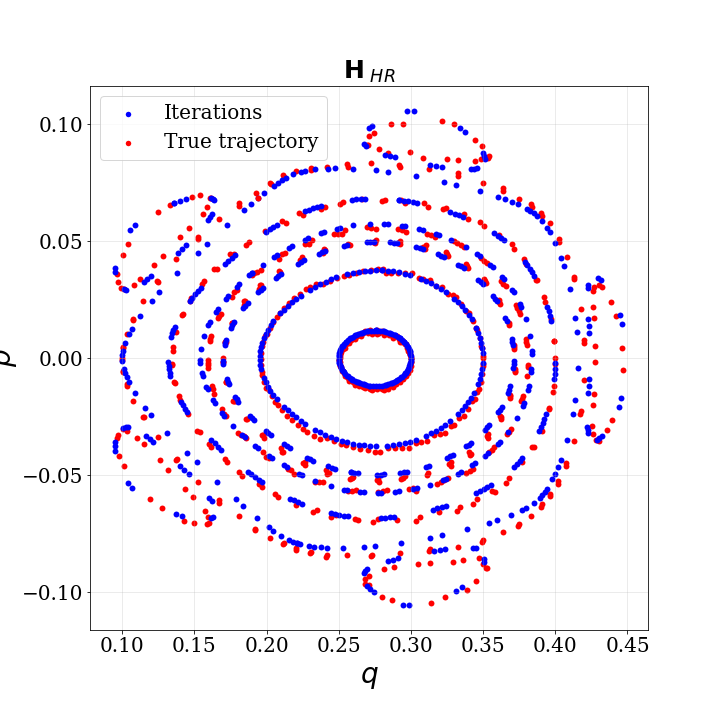}}\\
%\subfloat[fig 4]{\includegraphics[width = 3in]{}} 
\caption{\textbf{H\'enon-Heiles system}. Top: Poincar\'e plots obtained iterating the trained evolution maps from $\mathbf{H}_{SN}$,  $\mathbf{H}_{R}$, and $\mathbf{H}_{HR}$. Bottom: a few predicted trajectories.}
\label{PoincarePlots1}
\end{figure}

% \end{document}

\section{Conclusion}
% \documentclass[main.tex]{subfiles}
% \begin{document}
We have addressed the problem of learning the evolution map of a discrete-time dynamical system that is known to be either reversible or reversible symplectic. Many Hamiltonian systems (and their Poincar\'e maps) are reversible symplectic. We have proposed relevant hypothesis spaces with these structures in terms of trainable neural networks, extending the results of \cite{jin2020symplectic} and \cite{burby2020fast} to symplectic systems that are also reversible. We observed in numerical experiments that reversible symplectic networks perform better than existing symplectic ones, like SympNets. Also, even reversible networks without symplectic structure, involving non-symplectic Real NVP bijective layers, perform remarkably well on reversible symplectic systems.
%over SympNets, that are symplectic but not reversible, we conclude that reversibility is an important property for the learning algorithm to posses even if not paired with symplecticity.\\
% \end{document}

% Acknowledgments---Will not appear in anonymized version
%\acks{}
\newpage
\bibliography{Biblio}
\end{document}